\documentclass{article} 
\usepackage{times}
\usepackage{hyperref}
\usepackage{url}

\usepackage{amsthm}
\usepackage{amsmath}
\usepackage{amssymb}
\usepackage{dsfont}
\usepackage{booktabs}
\usepackage[pdftex]{graphicx}

\usepackage[final,multiuser]{fixme} 
\fxusetheme{colorsig}
\FXRegisterAuthor{mr}{amr}{MR}	
\FXRegisterAuthor{rf}{arf}{RF}	
\FXRegisterAuthor{bw}{abw}{BW}	
\FXRegisterAuthor{nm}{anm}{NM}	

\newcommand{\D}{\nabla}
\newcommand{\R}{\mathbb{R}}
\newcommand{\phat}{\hat{p}}
\newcommand{\regret}{\operatorname{Regret}}
\newcommand{\Mix}{\operatorname{Mix}}

\newcommand{\inner}[1]{\left\langle#1\right\rangle}

\newcommand{\E}[2]{\mathbb{E}_{#1}\left[#2\right]}
\newcommand{\Lag}{\mathcal{L}}
\newcommand{\ones}{\mathds{1}}

\newcommand{\interior}{\operatorname{int}}
\newcommand{\ri}{\operatorname{ri}}

\newcommand{\DT}{\Delta_\Theta}
\newcommand{\DX}{\Delta_X}
\newcommand{\acts}{\Ac}
\newcommand{\act}{a}
\newcommand{\acth}{\hat{\act}}

\newcommand{\lossx}{\ell}
\newcommand{\tlossx}{L}

\newcommand{\tsallis}{S}
\newcommand{\renyi}{R}

\newcommand{\Ac}{\mathcal{A}}

\DeclareMathOperator*{\argmin}{\arg\min}

\newcommand{\eg}{\textit{e.g.}}

\newcommand{\ie}{\textit{i.e.}}

\newtheorem{definition}{Definition}
\newtheorem{conjecture}{Conjecture}

\newtheorem{theorem}{Theorem}
\newtheorem{lemma}{Lemma}
\newtheorem{example}{Example}

\title{Generalized Mixability via Entropic Duality}

\author{
	Mark D. Reid\\
	Australian National University \& NICTA\\
	\and
	Rafael M. Frongillo\\
	Microsoft Research\\
	\and
	Robert C. Williamson\\
	Australian National University \& NICTA\\
	\and
	Nishant Mehta\\
	NICTA\\
}

\begin{document}

\maketitle

\begin{abstract}
	Mixability is a property of a loss which characterizes when fast
	convergence is possible in the game of prediction with expert
	advice. We show that a key property of mixability generalizes, and
	the exp and log operations present in the usual theory are not as
	special as one might have thought.
	In doing this we introduce a
	more general notion of $\Phi$-mixability where $\Phi$ is a general
	entropy (\ie, any convex function on probabilities). We show how a property 
	shared by the convex dual of any such entropy yields a natural
	algorithm (the minimizer of a regret bound) which, analogous to the
	classical aggregating algorithm, is guaranteed a constant regret
	when used with $\Phi$-mixable losses.
	We characterize precisely which $\Phi$ have $\Phi$-mixable losses
	and put forward a number of conjectures about the optimality and
	relationships between different choices of entropy.
\end{abstract}

\section{Introduction} 

\begin{amrnote}{}
	This intro was taken from COLT version -- I'm not sure the updating
	stuff is as relevant now.
	Rethink introduction in light of stronger results concerning 
	characterization and equivalence results.
	EXPLAIN ``Entropic Duality'' -- convex conjugation for functions
	restricted to a probability simplex -- see definition in \S\ref{sub:defn}.
\end{amrnote}

The combination or aggregation of predictions is central to machine
learning. 
Traditional Bayesian updating can be viewed as a particular way
of aggregating information that takes account of prior information.
Notions of ``mixability'' which play a key role in the setting of
prediction with expert advice offer a more general way to aggregate by
taking into account a loss function to evaluate predictions. 
As shown by Vovk~\cite{Vovk:2001}, his more general ``aggregating algorithm'' 
reduces to
Bayesian updating when log loss is used. 
However there is an implicit design variable in mixability that to date has
not been fully exploited.  
The aggregating algorithm makes use of a distance between the current
distribution and a prior which serves as a regularizer. 
In particular the aggregating algorithm uses the KL-divergence. 
We consider the general setting of an arbitrary loss and an arbitrary
regularizer (in the form of a Bregman divergence) and show that we recover
the core technical result of traditional mixability: if a loss is mixable in
our generalized sense then there is a generalized aggregating algorithm which
can be guaranteed to have constant regret. The generalized aggregating
algorithm is developed by optimizing the bound that defines our new notion
of mixability.
Our approach relies heavily on dual representations of entropy functions
defined on the probability simplex (hence the title). By doing so we
gain new insight into why the original mixability argument works and 
a broader understanding of when constant regret guarantees are possible.

\subsection{Mixability in Prediction With Expert Advice Games} 

A prediction with expert advice game is defined by its loss, a collection 
of experts that the player must compete against, and a fixed number of rounds.
Each round the experts reveal their predictions to the player and then
the player makes a prediction.
An observation is then revealed to the experts and the player and all
receive a penalty determined by the loss.
The aim of the player is to keep its total loss close to that of the best
expert once all the rounds have completed.
The difference between the total loss of the player and the total loss of
the best expert is called the regret and is the typically the focus of
the analysis of this style of game.
In particular, we are interested in when the regret is \emph{constant},
that is, independent of the number of rounds played.

More formally, let $X$ denote a set of possible \emph{observations}
and let $\acts$ denote a set of \emph{actions} or \emph{predictions}
the experts and player can perform.
A \emph{loss} $\lossx : \acts \to \R^X$ assigns the penalty 
$\ell_x(\act)$ to predicting $\act \in \acts$ when $x \in X$ is observed.
The set of experts is denoted $\Theta$ and the set of distributions over
$\Theta$ is denoted $\DT$.
In each round 
$t = 1, \ldots, T$, each expert $\theta\in\Theta$ makes a prediction
$\act^t_\theta \in \acts$.
These are revealed to the player who makes a prediction $\acth^t\in\acts$.
Once observation $x^t\in X$ is revealed the experts receive loss
$\lossx_{x^t}(\act^t_\theta)$ and the player receives loss 
$\lossx_{x^t}(\acth^t)$.
The aim of the player is to minimize its \emph{regret}
\bwnote{Made notation consistent here}
\(
	\regret(T) := \tlossx^T - \min_\theta \tlossx_\theta^T
\)
where $\tlossx^T := \sum_{t=1}^T \lossx_{x^t}(\acth^t)$ and 
$\tlossx_\theta^T = \sum_{t=1}^T \lossx_{x^t}(\act^t_\theta)$.
We will say the game has \emph{constant regret} if there exists a player
who can always make predictions that guarantee $\regret(T) \le R_{\ell,\Theta}$ 
for all $T$ and all expert predictions $\{\act^t_\theta\}_{t=1}^T$ where 
$R_{\ell,\Theta}$ is a constant that may depend on $\ell$ and $\Theta$.

In \cite{Vovk:1990, Vovk:1995}, Vovk showed that if the loss for a game
satisfies a condition called mixability then a player making 
predictions using the aggregating algorithm (AA) will achieve constant
regret.
\begin{definition}[Mixability and the Aggregating Algorithm]
	\label{def:vovk-mix}
	Given $\eta > 0$, a loss $\ell : \acts \to \R^X$ is \emph{$\eta$-mixable} 
	if, for all 
		expert predictions $\act_\theta \in \acts$, $\theta \in \Theta$ and all
		mixture distributions $\mu \in \DT$ over experts
	there exists 
		a prediction $\acth \in \acts$ such that 
	for all 
		outcomes $x \in X$
	we have
	\begin{equation}\label{eq:vovk-mix}
		\ell_x(\acth) 
		\le
		-\eta^{-1} \log \sum_{\theta \in \Theta} 
		\exp\left(
			-\eta \ell_x(\act_\theta)
		\right) \mu_\theta.
	\end{equation}
	The \emph{aggregating algorithm} starts with a mixture 
	$\mu^0 \in \DT$ over experts. In round $t$, experts predict
	$a^t_\theta$ and the player predicts 
	the $\acth^t \in \acts$ guaranteed by the $\eta$-mixability
	of $\ell$ so that \eqref{eq:vovk-mix} holds for $\mu = \mu^{t-1}$
	and $a_\theta = a^t_\theta$.
	Upon observing $x^t$, the mixture $\mu^t \in \DT$ is set so that
	$\mu_\theta^t 
	\propto \mu_\theta^{t-1} e^{-\eta \ell_{x^t}(a^t_\theta)}$.
\end{definition}

Mixability can be seen as a weakening of exp-concavity
(see \cite[\S3.3]{Cesa-Bianchi:2006}) that requires just enough of the loss
to ensure constant regret.

\begin{theorem}[Mixability implies constant regret \cite{Vovk:1995}]
	\label{thm:vovk-mix}
	If a loss $\ell$ is $\eta$-mixable then the aggregating algorithm
	will achieve $\regret(T) = \eta^{-1} \log |\Theta|$.
\end{theorem}


\subsection{Contributions} 

The key contributions of this paper are as follows. We provide a new
general definition (Definition~\ref{def:mix}) of mixability and an induced 
generalized aggregating algorithm (Definition~\ref{def:updates}) and show
(Theorem~\ref{thm:main}) that  prediction with expert advice using a
$\Phi$-mixable loss and the associated generalized aggregating algorithm is
guaranteed to have constant regret.   The proof illustrates that the log
and exp functions that arise in the classical aggregating algorithm are
themselves not special, but rather it is a translation invariant property
of the convex conjugate of and entropy $\Phi$ defined on a probability
simplex that is the crucial property that leads to constant regret.

We characterize
(Theorem~\ref{thm:legendre}) for which entropies $\Phi$ there exists
$\Phi$-mixable losses via the Legendre property.  
\bwnote{tweak if necessary} 
%
We show that
$\Phi$-mixability of a loss can be expressed directly in terms of the Bayes
risk associated with the loss (Definition~\ref{def:mix-F} and
Theorem~\ref{thm:mix-proper}), reflecting the situation that holds for
classical mixability~\cite{Erven:2012}.  As part of this analysis we show
that proper losses are quasi-convex (Lemma~\ref{lem:proper-qc}) which, to
the best of our knowledge appears to be a new result.


\subsection{Related Work} 

The starting point for mixability and the aggregating algorithm is the work
of \cite{Vovk:1995,Vovk:1990}.
The general setting of prediction with expert advice is summarized in
\cite[Chapters 2 and 3]{Cesa-Bianchi:2006}. There one can find a range of
 results that study different aggregation schemes and different
 assumptions on the losses (exp-concave, mixable).
 Variants of the aggregating algorithm have been studied for classically
 mixable losses, with a trade-off between tightness of the bound (in a
 constant factor) and the computational complexity \cite{Kivinen1999}.
Weakly mixable losses are a generalization of mixable losses. They have
been studied in \cite{Kalnishkan2008} where it is  shown there exists a
variant of the aggregating algorithm that achieves regret $C\sqrt{T}$ for
some constant $C$.
Vovk~\cite[in \S2.2]{Vovk:2001} makes the observation that his Aggregating
Algorithm 
reduces to Bayesian mixtures in the case of the log loss game. See also the
discussion in \cite[page 330]{Cesa-Bianchi:2006} relating certain
aggregation schemes to Bayesian updating.

The general form of updating we propose is similar to that considered by
Kivinen and Warmuth~\cite{Kivinen:1997} 
who consider finding a vector $w$ minimizing
\(
	d(w,s) + \eta L(y_t, w\cdot x_t)
\)
where $s$ is some starting vector, $(x_t, y_t)$ is the instance/label
observation at round $t$ and $L$ is a loss.  The key difference between
their formulation and ours is that our loss term is (in their notation)
$w\cdot L(y_t, x_t)$ -- \ie, the linear combination of the losses of the
$x_t$ at $y_t$ and not the loss of their inner product.
Online methods of density estimation for exponential families are discussed in
\cite[\S3]{Azoury:2001} where the authors compare the online and offline updates of
the same sequence and make heavy use of the relationship between the KL
divergence between members of an exponential family and an associated Bregman 
divergence between the parameters of those members.
The analysis of mirror descent \cite{Beck:2003} shows that it achieves 
constant regret when the entropic regularizer is used. 
However, there is no consideration regarding whether similar results 
extend to other entropies defined on the simplex.

We stress that the idea of the more general regularization and updates is
hardly new. See for example the discussion of potential based methods in
\cite{Cesa-Bianchi:2006} and other references later in the paper. The key
novelty is the generalized notion of mixability, the name of which is
justified by the key new technical result --- a constant regret bound
assuming the general mixability condition achieved via a generalized
algorithm which can be seen as intimately related to mirror descent.
Crucially, our result depends on some properties of the conjugates of potentials defined over probabilities that do not hold for potential functions defined over more general spaces.


\section{Generalized Mixability and Aggregation via Convex Duality} 

In this section we introduce our generalizations of mixability and the 
aggregating algorithm. 
One feature of our approach is the way the generalized aggregating algorithm
falls out of the definition of generalized mixability as the minimizer of 
the mixability bound.
Our approach relies on concepts and results from convex analysis.
Terms not defined below can be found in a reference such as 
\cite{Hiriart-Urruty:2001}.

\subsection{Definitions and Notation}\label{sub:defn} 

A convex function $\Phi : \DT \to \R$ is called an \emph{entropy} (on $\DT$)
if it is proper (\ie, $-\infty < \Phi \ne +\infty$), convex\footnote{
	While the information theoretic notion of Shannon entropy as a measure of
	uncertainty is concave, it is 
	convenient for us to work with convex functions on the simplex which
	can be thought of as certainty measures.
}, and lower semi-continuous.
In the following example and elsewhere we use $\ones$ to denote 
the vector $\ones_\theta = 1$ for all $\theta\in\Theta$ so that 
$|\Theta|^{-1} \ones \in \Delta_\Theta$ is the uniform distribution over 
$\Theta$. 
\begin{example}[Entropies]\label{ex:entropies}
	The \emph{(negative) Shannon entropy} 
	$H(\mu) := \sum_\theta \mu_\theta \log \mu_\theta$;
	the \emph{quadratic entropy}
	$Q(\mu) := \sum_\theta (\mu - |\Theta|^{-1}\ones)^2$;
	the \emph{Tsallis entropies} 
	$\tsallis_\alpha(\mu) := \alpha^{-1} \left(
		\sum_\theta \mu_\theta^{\alpha+1} - 1
	\right)$ for $\alpha \in (-1,0) \cup (0, \infty)$;
	and the \emph{R\'enyi entropies}
	$\renyi_\alpha(\mu) = \alpha^{-1} \left( 
		\log \sum_\theta \mu_\theta^{\alpha + 1}
	\right)$, for $\alpha \in (-1, 0)$. 
	We note that both Tsallis and R\'enyi entropies limit to Shannon 
	entropy $\alpha \to 0$ (cf. \cite{Maszczyk:2008,Van-Erven:2012}).
\end{example}
Let $\inner{\mu, v}$ denote the inner product between 
$\mu \in \DT$ and $v\in\DT^*$, the dual space of $\DT$.
The \emph{Bregman divergence} associated with a suitably differentiable 
entropy $\Phi$ on $\DT$ is given by
\begin{equation}
	\label{eq:bregman-def}
	D_\Phi(\mu, \mu')
	= 
	\Phi(\mu) - \Phi(\mu') - \inner{\mu - \mu', \D\Phi(\mu')}
\end{equation}
for all $\mu \in \DT$ and $\mu' \in \ri(\DT)$, the relative interior of
$\DT$.
Given an entropy $\Phi : \DT \to \R$, we define its \emph{entropic dual} 
to be $\Phi^*(v) := \sup_{\mu\in\DT} \inner{\mu,v} - \Phi(\mu)$ where 
$v \in \DT^*$, \ie, the dual space to $\DT$.
Note that one could also write the supremum over 
$\R^{\Theta}$ by setting $\Phi(\mu) = +\infty$ for $\mu \notin\DT$ so
that $\Phi^*$ is just the usual convex dual (cf. \cite{Hiriart-Urruty:2001}).
Thus, all of the standard results about convex duality also 
hold for entropic duals provided some care is taken with the domain of 
definition.
We note that although the regular convex dual of $H$ defined over all of $\R^\Theta$ is $v \mapsto \sum_\theta \exp(v_\theta-1)$ its entropic dual is 
$H^*(v) = \log\sum_\theta \exp(v_\theta)$.

For differentiable $\Phi$, it is known \cite{Hiriart-Urruty:2001} that the 
supremum defining $\Phi^*$ is attained at $\mu = \D\Phi^*(v)$.
That is,
\begin{equation}\label{eq:conjugate-sup}
	\Phi^*(v) = \inner{\D\Phi^*(v),v} - \Phi(\D\Phi^*(v)).
\end{equation}
A similar result holds for $\Phi$ by applying this result to $\Phi^*$ and
using $\Phi = (\Phi^*)^*$.
We will make repeated use of two easy established properties of entropic duals 
(see Appendix~\ref{app:lemmas} for proof).
\begin{lemma}\label{lem:prob-duals}
	If $\Phi$ is an entropy over $\DT$ and $\Phi_\eta := \eta^{-1}\Phi$ denotes
	a scaled version of $\Phi$ then 1) for all $\eta > 0$ we have
	$\Phi_\eta^*(v) = \eta^{-1}\Phi^*(\eta v)$; and 2) the entropic dual 
	$\Phi^*$ is \emph{translation invariant} -- \ie, for all
	$v \in \DT^*$ and $\alpha \in \R$ we have 
	$\Phi^*(v + \alpha\ones) = \Phi^*(v) + \alpha$ and hence for differentiable
	$\Phi^*$ we have
	$\D\Phi^*(v + \alpha\ones) = \D\Phi^*(v)$.
\end{lemma}

The translation invariance if $\Phi^*$ is central to our analysis. 
It is what ensures our $\Phi$-mixability inequality \eqref{eq:mixability} 
``telescopes'' when it is summed.
The proof of the original mixability result (Theorem~\ref{thm:vovk-mix}) uses 
a similar telescoping argument that works due to the interaction of $\log$
and $\exp$ terms in Definition~\ref{def:vovk-mix}.
Our results show that this telescoping property is not due to any special
properties of $\log$ and $\exp$, but rather because of the translation
invariance of the entropic dual of Shannon entropy, $H$.
The following analysis generalizes that of the original work on mixability 
precisely because this property holds for the dual of any entropy.


\subsection{$\Phi$-Mixability and the Generalized Aggregating Algorithm} 

For convenience, we will use $A \in \acts^{\Theta}$ to denote a collection of 
expert predictions and $A_\theta \in \acts$ to denote the prediction of 
expert $\theta$.
Abusing notation slightly, we will write $\ell(A) \in \R^{X\times\Theta}$ for 
the matrix of loss values $[ \ell_x(A_\theta) ]_{x,\theta}$, 
and $\ell_x(A) = [\ell_x(A_\theta)]_\theta \in \R^\Theta$ for the vector of 
losses for each expert $\theta$ on outcome $x$.

\begin{definition}[$\Phi$-mixability]\label{def:mix}
	Let $\Phi$ be an entropy on $\DT$. A loss $\ell : \acts \to \R^X$ is
	$\Phi$-mixable if for all $A \in \acts^{\Theta}$, all $\mu\in\DT$,
	there exists an $\acth \in \acts$ such that for all $x \in X$
	\begin{equation}\label{eq:mixability}
		\ell_x(\acth) 
		\le 
		\Mix_{\ell,x}^\Phi(A,\mu) 
		:= \inf_{\mu'\in\DT} \inner{\mu', \ell_x(A)} + D_\Phi(\mu', \mu).
	\end{equation}
\end{definition}

The term on the right-hand side of \eqref{eq:mixability} has some intuitive
appeal.
Since $\inner{\mu',A} = \E{\theta \sim \mu'}{\ell_x(A_\theta)}$ 
(\ie, the expected loss 
of an expert drawn at random according to $\mu'$) we can view the 
optimization as a trade off between finding a mixture $\mu'$ that tracks the 
expert with the smallest loss upon observing outcome $x$ and keeping $\mu'$ 
close to $\mu$, as measured by $D_\Phi$.
In the special case when $\Phi$ is Shannon entropy, $\ell$ is log loss, and 
expert predictions $A_\theta \in \DX$ are distributions over $X$ such an
optimization is equivalent to Bayesian updating \cite{Williams:1980}.

To see that $\Phi$-mixability is indeed a generalization of 
Definition~\ref{def:vovk-mix}, we
make use of an alternative form for the right-hand side of
the bound in the $\Phi$-mixability definition that ``hides'' the infimum 
inside $\Phi^*$.
As shown in Appendix~\ref{app:lemmas} this is a straight-forward consequence of
\eqref{eq:conjugate-sup}.

\begin{lemma}\label{lem:alt-mix}
	The mixability bound 
	\begin{equation}\label{eq:alt-mix}
		\Mix_{\ell,x}^\Phi(A,\mu) 
		= 
		\Phi^*(\D\Phi(\mu)) - \Phi^*(\D\Phi(\mu) - \ell_x(A)).
	\end{equation}
	Hence, for $\Phi = \eta^{-1}H$ we have
	$\Mix_{\ell,x}^\Phi(A,\mu) 
	= -\eta^{-1}\log\sum_\theta\exp(-\eta \ell_x(A_\theta))\mu_\theta$
	which is the bound in Definition~\ref{def:vovk-mix}.
\end{lemma}

We now define a generalization of the Aggregating Algorithm  of
Definition~\ref{def:vovk-mix} that very naturally relates to our 
definition of $\Phi$-mixability:
starting with some initial distribution over experts, the algorithm
repeatedly incorporates the information about the experts' performances
by finding the minimizer $\mu'$ in \eqref{eq:mixability}.

\begin{definition}[Generalized Aggregating Algorithm]\label{def:updates}
	The algorithm begins with a mixture distribution $\mu^0 \in \DT$ over
	experts.
	On round $t$, after receiving expert predictions $A^t \in \acts^\Theta$,
	the \emph{generalized aggregating algorithm} (GAA) predicts 
	any $\acth \in \acts$ 
	such that $\ell_x(\acth) \le \Mix_{\ell,x}^\Phi(A^t,\mu^{t-1})$ 
	for all $x$ which is 
	guaranteed to exist by the $\Phi$-mixability of $\ell$.
	After observing $x^t \in X$, the GAA updates the mixture 
	$\mu^{t-1} \in \DT$ by setting
	\begin{equation}
		\label{eq:gen-aa-update}
		\mu^t 
		:= 
		\argmin_{\mu' \in \DT}
			\inner{\mu',\ell_{x^t}(A^t)} + D_\Phi(\mu', \mu^{t-1}).
	\end{equation}
\end{definition}

We now show that this updating process simply aggregates the 
per-expert losses $\ell_x(A)$ in the dual space 
$\DT^*$ with $\D\Phi(\mu^0)$ as the starting point.
The GAA is therefore closely related to mirror descent techniques
\cite{Beck:2003}.

\begin{lemma}\label{lem:updates}
	The GAA updates $\mu^t$ in \eqref{eq:gen-aa-update} satisfy
	\(
		\D\Phi(\mu^t) 
		=
		\D\Phi(\mu^{t-1}) - \ell_{x^t}(A^t)
	\)
	for all $t$ and so
	\begin{equation}\label{eq:updates}
		\D\Phi(\mu^T) = \D\Phi(\mu^0) - \sum_{t=1}^T \ell_{x^t}(A^t).
	\end{equation}
\end{lemma}

The proof is given in Appendix~\ref{app:lemmas}.
Finally, to see that the above is indeed a generalization of the Aggregating 
Algorithm from Definition~\ref{def:vovk-mix} we need only apply 
Lemma~\ref{lem:updates} and observe that for 
$\Phi = \eta^{-1} H$ we have 
$\D\Phi(\mu) = \eta^{-1}(\log(\mu) + \ones)$ and so 
$\log \mu^t = \log \mu^{t-1} - \eta \ell_{x^t}(A^t)$.
Exponentiating this vector equality element-wise gives 
$\mu_\theta^t \propto \mu^{t-1}_\theta \exp(-\eta \ell_{x^t}(A_\theta^t))$.



\section{Properties of $\Phi$-mixability} 

In this section we establish a number of key properties for $\Phi$-mixability,
the most important of these being that $\Phi$-mixability implies constant
regret.
We also show that $\Phi$-mixability is not a vacuous concept for $\Phi$ other
than Shannon entropy by showing that any Legendre $\Phi$ has $\Phi$-mixable 
losses and that this is a necessary condition for such losses to exist.

\subsection{$\Phi$-mixability Implies Constant Regret} 

\begin{theorem}\label{thm:main}
	If $\lossx : \acts \to \R^X$ is $\Phi$-mixable then there is
	a family of strategies parameterized by $\mu\in\DT$ which,
	for any sequence of observations 
	$x^1, \ldots, x^T \in X$ 
	and sequence of expert predictions $A^1, \ldots , A^T \in \acts^\Theta$, 
	plays a sequence $\acth^1, \ldots, \acth^T \in \acts$ such that for all
	$\theta \in \Theta$
	\begin{equation}\label{eq:main}
		\sum_{t=1}^T \lossx_{x^t}(\acth^t)
		\le
		\sum_{t=1}^T \lossx_{x^t}(A^t_\theta) +
			D_\Phi(\delta_\theta, \mu) .
	\end{equation}
\end{theorem}

The proof is in Appendix~\ref{app:proof-of-main} and is a straight-forward
consequence of Lemma~\ref{lem:alt-mix} and the translation invariance of
$\Phi^*$.
The standard notion of mixability is recovered when $\Phi = \frac{1}{\eta}H$ 
for $\eta > 0$ and $H$ the Shannon entropy on $\DT$.
In this case, Theorem~\ref{thm:vovk-mix} is obtained as a corollary
for $\mu = |\Theta|^{-1}\ones$, the uniform distribution
over $\Theta$.
A compelling feature of our result is that it gives a natural interpretation
of the constant $D_\Phi(\delta_\theta, \pi)$ in the regret bound: if
$\pi$ is the initial guess as to which expert is best before the game starts,
the ``price'' that is paid by the player is exactly how far (as measured by
$D_\Phi$) the initial guess was from the distribution that places all its
mass on the best expert.

The following example computes mixability bounds for the alternative
entropies introduced in \S\ref{sub:defn}.
They will be discussed again in \S\ref{sub:asymptotics} below.

\begin{example}\label{ex:regrets}
	Consider games with $K = |\Theta|$ experts and $\mu = K^{-1}\ones$.
	For the (negative) Shannon entropy, 
	the regret bound from Theorem~\ref{thm:main} is 
	$D_H(\delta_\theta, \mu) = \log K$.
	For quadratic entropy the regret bound is 
	$D_Q(\delta_\theta, \mu) = 1 - \frac{2(K-1)}{K^2}$.
	For the family of Tsallis entropies 
	the regret bound 
	given by 
	$D_{\tsallis_\alpha}(\delta_\theta, K^{-1}\ones) = \alpha^{-1}(1-K^{-\alpha})$.
	For the family of R\'enyi entropies 
	 the regret bound becomes
	$D_{\renyi_\alpha}(\delta_\theta, K^{-1}\ones) = \log K$.
\end{example}

A second, easily established result concerns the mixability of scaled entropies.
The proof is by observing that in \eqref{eq:mixability} the only term in the 
definition of 
$\Mix^{\Phi_\eta}_{\ell,x}$ involving $\eta$ is 
$D_{\Phi_\eta} = \frac{1}{\eta}D_\Phi$.
The quantification over $A,\mu,\acth, \mu'$ and $x$ in the original definition
have been translated into infima and suprema.
\begin{lemma}\label{lem:scaling}
	The function
	\(
		M(\eta) 
		:=
		\inf_{A,\mu} \sup_{\acth} \inf_{\mu', x} \;
		\Mix^{\Phi_\eta}_{\ell,x}(A,\mu) - \ell_x(\acth)
	\)
	is non-increasing.
\end{lemma}

This implies that there is a well-defined maximal $\eta > 0$ for which a given
loss $\ell$ is $\Phi_\eta$-mixable since $\Phi_\eta$-mixability is equivalent
to $M(\eta) \ge 0$.
We will call this maximal $\eta$ the \emph{$\Phi$-mixability constant} for $\ell$ 
and denote it $\eta(\ell,\Phi) := \sup \{ \eta > 0 : M(\eta) \ge 0 \}$. 
This constant is central to the discussion in Section~\ref{sub:regret} below.



\subsection{$\Phi$-Mixability of Proper Losses and Their Bayes Risks} 

Entropies are known to be closely related to the Bayes risk of what are
called proper losses or proper scoring rules~\cite{Dawid:2007,Gneiting:2007}.
Here, the predictions are distributions over outcomes, \ie, points in $\DX$.
To highlight this we will use $p$, $\phat$ and $P$ instead of $a$, $\acth$ 
and $A$ to denote actions.
If a loss $\ell : \DX \to \R^X$ is used to assign a
penalty $\ell_x(\phat)$ to a prediction $\phat$ upon outcome $x$
it is said to be \emph{proper} if its expected value under $x\sim p$
is minimized by predicting $\phat = p$. 
That is, for all $p, \phat \in \DX$ we have
\[
	\E{x\sim p}{\ell_x(\phat)}
	=
	\inner{p,\ell(\phat)} 
	\ge
	\inner{p,\ell(p)}
	=: -F^\ell(p)
\]
where $-F^\ell$ is the \emph{Bayes risk} of $\ell$ and is necessarily
concave~\cite{Erven:2012}, thus making $F^\ell : \DX\to\R$ convex
and thus an entropy.
The correspondence also goes the other way: given any convex function 
$F : \DX\to\R$ we can construct a unique proper loss~\cite{Vernet:2011}.
The following representation can be traced back to ~\cite{Savage:1971} but
is expressed here using convex duality.

\begin{lemma}\label{lem:proper-loss}
	If $F:\DX\to\R$ is a differentiable entropy then the loss
	$\ell^F :\DX\to\R$ defined by
	\begin{equation}\label{eq:proper-loss}
		\ell^F(p)
		:=
		F^*(\D F(p))\ones - \D F(p)
	\end{equation}
	is proper.
\end{lemma}

It is straight-forward to show that the proper loss associated with the
negative Shannon entropy $\Phi = H$ is the log loss, that is,
$\ell^{H}(\mu) := -\left(\log \mu(\theta)\right)_{\theta\in\Theta}$.

This connection between losses and entropies lets us define the 
$\Phi$-mixability of a proper loss strictly in terms of its associated entropy.
This is similar in spirit to the result in \cite{Erven:2012} which shows that
the original mixability (for $\Phi = H$) can be expressed in terms of the
relative curvature of Shannon entropy and the loss's Bayes risk.
We use the following definition to explore the optimality of Shannon mixability
in Section~\ref{sub:regret} below.

\begin{definition}\label{def:mix-F}
	An entropy $F : \Delta_X \to \R$ is \emph{$\Phi$-mixable} if
	\begin{equation}\label{eq:mix-F}
		\sup_{P,\mu} \;
			F^*\left(
				\left\{ \Phi^*(\D\Phi(\mu) - \ell^F_x(P)) \right\}_x
				- \Phi^*(\D\Phi(\mu))\ones
			\right)
		\le 0
	\end{equation}
	where $\ell^F$ is as in Lemma~\ref{lem:proper-loss} and the supremum is 
	over expert predictions $P \in \DX^\Theta$ and
	mixtures over experts $\mu\in\DT$.
\end{definition}

Although this definition appears complicated due to the handling of vectors
in $\R^X$ and $\R^\Theta$, it has a natural interpretation in terms of 
\emph{risk measures} from mathematical finance \cite{Follmer:2004}.
Given some convex function $\alpha : \Delta_X \to \R$, its associated risk
measure is its dual $\rho(v) := \sup_{p\in\DX} \inner{p,-v} - \alpha(p) = \alpha^*(-v)$ where $v$ is a \emph{position} meaning $v_x$ is some monetary value 
associated with outcome $x$ occurring.
Due to its translation invariance, the quantity $\rho(v)$ is often interpreted as 
the amount of ``cash'' (\ie, outcome independent value) an agent would ask for to 
take on the uncertain position $v$. 
Observe that the risk $\rho^F$ for when $\alpha = F$ satisfies 
$\rho^F\circ\ell^F = 0$ so that $\ell^F(p)$ is always a $\rho^F$-risk free
position.
If we now interpret $\mu^* = \D\Phi(\mu)$ as a position over outcomes in 
$\Theta$ and $\Phi^*$ as a risk for $\alpha = \Phi$ the term 
$\left\{ \Phi^*(\mu^* - \ell^F_x(P)) \right\}_x - \Phi^*(\mu^*)\ones$
can be seen as the change in $\rho^\Phi$ risk when shifting position 
$\mu^*$ to $\mu^* - \ell^F_x(P)$ for each possible outcome $x$.
Thus, the mixability condition in \eqref{eq:mix-F} can be viewed as a
requirement that a $\rho^F$-risk free change in positions over $\Theta$ always 
be $\rho^F$-risk free.

The following theorem shows that the entropic version of $\Phi$-mixability 
Definition~\ref{def:mix-F} is equivalent to the loss version in Definition~\ref{def:mix} in the case of proper losses.
Its proof can be found in Appendix~\ref{app:mix-proper}
and relies on Sion's theorem and the facts that proper losses are quasi-convex. 
This latter fact appears to be new so we state it here as a separate lemma
and prove it in Appendix~\ref{app:lemmas}.

\begin{lemma}\label{lem:proper-qc}
	If $\ell : \DX \to \R$ is proper then $p' \mapsto \inner{p, \ell(p')}$
	is quasi-convex for all $p \in \DX$.
\end{lemma}

\begin{theorem}\label{thm:mix-proper}
	If $\ell : \DX \to \R^X$ is proper and has Bayes risk $-F$ then $F$ is an
	entropy and $\ell$ is $\Phi$-mixable if and only if $F$ is $\Phi$-mixable.
\end{theorem}

The entropic form of mixability in \eqref{eq:mix-F} shares some similarities with 
expressions for
the classical mixability constants given in \cite{Haussler:1998} for binary 
outcome games and in \cite{Erven:2012} for general games.
Our expression for the mixability is more general than the previous two 
being both for binary and non-binary outcomes and for general entropies.
It is also arguably more efficient since the optimization in \cite{Erven:2012}
for non-binary outcomes requires inverting a Hessian matrix at each point in the 
optimization.


\subsection{Characterizing and Comparing $\Phi$-mixability} 

Although Theorem~\ref{thm:main} recovers the already known constant regret 
bound for Shannon-mixable losses, it is natural to ask whether the result is 
vacuous or
not for other entropies. 
That is, do there exist $\Phi$-mixable losses for $\Phi$ other than Shannon
entropy?
If so, do such $\Phi$-mixable losses exist for any entropy $\Phi$?
The next theorem answers both of these questions, showing that the existence
of ``non-trivial'' $\Phi$-mixable losses is intimately related to the behaviour 
of an entropy's gradient at the simplex's boundary.
Specifically, an entropy $\Phi$ is said to be \emph{Legendre} 
\cite{Rockafellar:1997} if: a) $\Phi$ is strictly convex  
in $\interior(\Delta_\Theta)$; and b) 
$\|\D \Phi(\mu)\| \to \infty$ as $\mu \to \mu_b$ for any $\mu_b$ on the boundary
of $\Delta_\Theta$.

We will say a loss is \emph{non-trivial} if there exist distinct actions which
are optimal for distinct outcomes (see \ref{app:legendre} for formal definition).
This, for example, rules out constant losses -- \ie, $\ell(a) = k \in \R^X$ for all $a\in\acts$ -- are easily\footnote{
	The inequality in \eqref{eq:mixability} reduces to 
	$0 \le \inf_{\mu'} D_\Phi(\mu',\mu)$ which is true for all 
	Bregman divergences.
}
seen to be $\Phi$-mixable for any $\Phi$.
For technical reasons we will further restrict our attention to \emph{curved} 
losses by which we mean those losses with strictly concave Bayes risks.
We conjecture that the following theorem also holds for non-curved losses.

\begin{theorem}\label{thm:legendre}
	There exist non-trivial, curved $\Phi$-mixable losses 
	if and only if 
	the entropy $\Phi$ is Legendre.
\end{theorem}

The proof is in Appendix~\ref{app:legendre}.
From this result we can deduce that there are no $Q$-mixable
losses. Also, since it is easy to show the derivatives $\D S_\alpha$ and $\D R_\alpha$ are unbounded for $\alpha \in (0,1)$, the entropies $S_\alpha$ and 
$R_\alpha$ are Legendre. Thus there exist $S_\alpha$- and $R_\alpha$-mixable  losses when $\alpha\in(-1,0)$.

\section{Conclusions and Open Questions} 

The main purpose of this work was to shed new light on mixability by 
casting it within the broader notion of $\Phi$-mixability.
We showed that the constant regret bounds enjoyed by mixability losses
are due to the translation invariance of entropic duals, and so are also 
enjoyed by any $\Phi$-mixable loss.
The definitions and technical machinery presented here allow us to ask
precise questions about entropies and the optimality of their associated 
aggregating algorithms.

\begin{amrnote}{}
	Summarize findings before leading into open questions.
\end{amrnote}
\begin{amrnote}{}
	Emphasize this work presenting mixability in a new light via convexity.
	Argue that it gives a much fuller picture of when constant regret bounds
	are possible and shows a strong connection between constant regret
	and Legendre functions (cf. barrier methods?)
\end{amrnote}

\subsection{Are All Legendre Entropies ``Equivalent''?}

Since Theorem~\ref{thm:legendre} shows the existence of $\Phi$-mixable losses, a natural question concerns the relationship between the sets of losses that
are mixable for different choices of $\Phi$.
For example, are there losses that are $H$-mixable but not $S_\alpha$-mixable,
or vice-versa?
We conjecture that essentially all Legendre entropies $\Phi$ 
have the same $\Phi$-mixable losses up to a scaling factor.

\begin{amrnote}{}
	Prove following or state as conjecture. I believe it is true.
	It is equivalent to showing that 
	$\Mix^{\Phi}_{\ell,x}(A_\eta, \mu_\eta)
	> 
	\eta^{-1}\Mix^{\Phi'}_{\eta\ell,x}(A_\eta, \mu_\eta)$
	yields a contradiction
	where $A_\eta$ and $\mu_\eta$ are the witnesses to the
	non-$\Phi'_\eta$-mixability of $\ell$ -- that is,
	$\inf_{A,\mu} \sup_{\acth} \inf_{x} \Mix^{\Phi'_\eta}_{\ell,x}(A,\mu) < 0$.
	Since the RHS looks like a directional derivative of $\Phi^*$, I expect
	the upper bound on the LHS should violate the Legendre assumption 
	on $\Phi'$.
	However, the analysis is made tricky by the dependence of $A_\eta$ and 
	$\mu_\eta$ on $\eta$ as $\eta \to 0$.
\end{amrnote}

\begin{conjecture}\label{con:equivalent}
	Let $\Phi$ be a entropy on $\DT$ and $\ell$ be a $\Phi$-mixable loss.
	If $\Psi$ is a Legendre entropy on $\DT$ then there
	exists an $\eta > 0$ such that $\ell$ is $\eta^{-1}\Psi$-mixable.
\end{conjecture}

Some intuition for this conjecture is derived from observing that 
$\Mix^{\Psi_\eta}_{\ell,x} = \eta^{-1} \Mix^\Psi_{\eta\ell,x}$ and that 
as $\eta \to 0$ the function $\eta\ell$ behaves like a constant loss and 
will therefore be mixable. This means that scaling up $\Mix^\Psi_{\eta\ell,x}$
by $\eta^{-1}$ should make it larger than $\Mix^\Phi_{\ell,x}$.
However, some subtlety arises in ensuring that this dominance occurs uniformly.

\subsection{Asymptotic Behaviour}\label{sub:asymptotics} 

There is a lower bound due to Vovk \cite{Vovk:1995} for general losses
$\ell$ which shows that if one is allowed to vary the number of rounds
$T$ and the number of experts $K = |\Theta|$, then no regret bound can
be better than the optimal regret bound obtained by Shannon
mixability. Specifically, for a fixed loss $\ell$ with optimal Shannon
mixability constant $\eta_\ell$, suppose that for some $\eta' >
\eta_\ell$ we have a regret bound of the form $(\log K)/ \eta'$
as well as some strategy $L$ for the learner that supposedly
satisfies this regret bound. Vovk's lower bound shows, for this
$\eta'$ and $L$, that there exists an instantiation of the prediction
with expert advice game with $T$ large enough and $K$  roughly exponential in $T$ (and both are still finite) for which the alleged regret bound will fail to hold at the
end of the game with non-zero probability.
The regime in which Vovk's lower bound holds suggests that the best achievable regret with respect to the number of experts grows as $\log K$. Indeed, there is a lower bound for general losses $\ell$ that shows the regret of the best possible algorithm on games using $\ell$ must grow like $\Omega(\log_2 K)$ \cite{Haussler:1998}.

The above lower bound arguments apply when the number of experts is
large (\ie, exponential in the number of rounds) or if we consider the
dynamics of the regret bound as $K$ grows. This leaves open the
question of the best possible regret bound for moderate and possibly
fixed $K$ which we formally state in the next section.
This question that serves as a strong motivation for the study of generalized mixability considered here. Note also that the above lower bounds are consistent with the fact that there cannot be non-trivial, $\Phi$-mixable losses for non-Legendre $\Phi$ (\eg, the quadratic entropy $Q$) since the growth of the regret bound as a function of $K$ 
(cf. Example~\ref{ex:regrets}) is less than $\log K$ and hence violates the above lower bounds.




\subsection{Is There An ``Optimal'' Entropy?} \label{sub:regret} 

Since we believe that $\Phi$-mixability for Legendre $\Phi$ yield the same 
set of losses, we can ask whether, for a fixed loss $\ell$,
some $\Phi$ give better regret bounds than others.
These bounds depend jointly on the largest $\eta$ such that 
$\ell$ is $\Phi_\eta$-mixable and the value of $D_\Phi(\delta_\theta, \mu)$.
We can define the optimal regret bound
one can achieve for a particular loss $\ell$ using the generalized aggregating
algorithm with $\Phi_\eta := \tfrac 1 \eta \Phi$ for some $\eta > 0$.  This
allows us to compare entropies on particular losses, and we can say that an
entropy \emph{dominates} another if its optimal regret bound is better for all
losses $\ell$.
Recalling the definition of the maximal $\Phi$-mixability constant from Lemma~\ref{lem:scaling}, we can determine a quantity of more direct interest: 
the best regret bound one can obtain using a scaled copy of $\Phi$.  
Recall that if $\ell$ is $\Phi$-mixable, then the best regret bound one can
achieve from the generalized aggregating algorithm is $\inf_{\mu}\sup_{\theta}
D_\Phi(\delta_\theta,\mu)$.  
We can therefore define the best regret bound for $\ell$ on a scaled version of
$\Phi$ to be 
$R_{\ell,\Phi} := 
	\eta(\ell,\Phi)^{-1} \inf_{\mu}\sup_{\theta} D_\Phi(\delta_\theta,\mu)$ 
which simply corresponds to the regret bound for the entropy 
$\Phi_{\eta(\ell,\Phi)}$.  
Note a crucial property of $R_{\ell,\Phi}$, which will be very useful in 
comparing entropies: $R_{\ell,\Phi} = R_{\ell,\alpha\Phi}$ for all $\alpha>0$.  
(This follows from the observation that 
$\eta(\ell,\alpha\Phi) = \eta(\ell,\Phi)/\alpha$.)  
That is, $R_{\ell,\Phi}$ is independent of the particular scaling we 
choose for $\Phi$.

We can now use $R_{\ell,\Phi}$ to define a scale-invariant relation over entropies.  
Define
$\Phi \geq_\ell \Psi$ if $R_{\ell,\Phi} \leq R_{\ell,\Psi}$, and $\Phi
\geq_* \Psi$ if $\Phi \geq_\ell \Psi$ for all losses $\ell$.  
In the latter case we say 
$\Phi$
\emph{dominates} $\Psi$.
By construction, if one entropy dominates another its regret bound is guaranteed
to be tighter and therefore its aggregating algorithm will achieve better 
worst-case regret.
As discussed above, one natural candidate for a universally dominant entropy is the Shannon entropy.

\begin{conjecture}
  \label{quest:msr-1}
  For all choices of $\Theta$, the negative Shannon entropy dominates all 
  other entropies. That is, $H \geq_* \Phi$ for all
  $\Theta$ and all convex $\Phi$ on $\DT$.
\end{conjecture}

Although we have not been able to prove this conjecture we were able to collect
some positive evidence in the form of Table~\ref{tab:regrets}.
Here, we took the entropic form of $\Phi$-mixability from 
Definition~\ref{def:mix-F} and implemented\footnote{
	In order to preserve anonymity the code will not be made available 
	until after publication.
}
it as an optimization problem in 
the language R and computed $\eta(\ell^F, \Phi)$ for $F$ and $\Phi$ equal to the entropies introduced in Example~\ref{ex:entropies} for two expert games with two outcomes. 
The maximal $\eta$ (and hence the optimal regret bounds)  for each pair was found doing a binary search
for the 
zero-crossing of $M(\eta)$ from Lemma~\ref{lem:scaling} and then applying the 
bounds from Example~\ref{ex:regrets}.
Although we were expecting the dominant entropy for each loss $\ell^F$ to be 
its ``matching'' entropy (\ie, $\Phi = F$), as can be seen from the table the 
optimal regret bound for every loss was obtained in the column for $H$.
However, one interesting feature for these matching cases is that the optimal
$\eta$ (shown in parentheses) is always equal to 1.

\begin{conjecture}
	Suppose $|X| = |\Theta|$ so that $\DT = \DX$. 
	Given a Legendre $\Phi : \DT \to \R$ and its associated proper loss 
	$\ell^\Phi : \DX \to \R^X$, the 
	maximal $\eta$ such that $\ell^\Phi$ is $\eta^{-1}\Phi$-mixable is 
	$\eta = 1$.
\end{conjecture}

We conjecture that this pattern will hold for matching entropies and losses
for larger numbers of experts and outcomes and hope to test or prove this in future work.

\begin{table} \centering
\caption{Mixability and optimal regrets for pairs of losses and entropies in
	2 outcome/2 experts games. 
	Entries show the regret bound 
	$\eta^{-1}D_{\Phi}(\delta_\theta, \frac{1}{2}\ones)$
	for the maximum $\eta$ (in parentheses).
	\label{tab:regrets}}
\begin{tabular}{@{} lccccccccc @{}}  
	& \multicolumn{9}{c}{\textbf{Entropy}} 
\\ \cmidrule{2-10} 
	\textbf{Loss}
	& $H$
	& 
	& $S_{-.1}$ 
	& $S_{-.5}$ 
	& $S_{-.9}$ 
	& 
	& $R_{-.1}$ 
	& $R_{-.5}$ 
	& $R_{-.9}$
\\ \midrule  
    $\log$   		
    & 0.69 ($1^*$) 
    &
    & 0.74 (.97)  
	& 1.17 (.71) 
	& 5.15 (.19)
    &
    & 0.77 (0.9) 
	& 1.38 (0.5)
	& 6.92 (0.1) 
\\
    $\ell^Q$   		
    & 0.34 (2) 
    &
    & 0.37 (1.9)  
	& 0.58 (1.4)
	& 2.57 (0.4)
    &
    & 0.38 (1.8)
	& 0.69 (1)
	& 3.45 (0.2) 
\\
    $\ell^{S_{-.5}}$   		
    & 0.49 (1.4) 
    &
    & 0.53 (1.4) 
	& 0.82 ($1^*$)
	& 3.64 (.26)
    &
    & 0.54 (1.3)
	& 0.98 (.71)
	& 4.90 (.14)
\\
    $\ell^{R_{.5}}$   		
    & 0.34 (2) 
    &
    & 0.37 (1.9)  
	& 0.58 (1.4) 
	& 2.57 (.37) 
    &
    & 0.38 (1.8) 
	& 0.69 ($1^*$) 
	& 3.46 (0.2) 
 \\\bottomrule
\end{tabular}
\end{table}





\subsubsection*{Acknowledgments}

We would like to thank Matus Telgarsky for help with restricted duals, Brendan 
van Rooyen for noting that there are no quadratic mixable losses, and Harish 
Guruprasad for identifying a flaw in an earlier ``proof'' of the quasi-convexity 
of proper losses.
Mark Reid is supported by an ARC Discovery Early Career Research Award (DE130101605) and part of this work was developed while he was visiting Microsoft Research.
NICTA is funded by the Australian Government and as an ARC ICT Centre of Excellence.

\newpage


\bibliographystyle{unsrt}
\bibliography{phi-mix}

\newpage

\appendix

\section{Appendix} 

\subsection{Proof of Lemmas} \label{app:lemmas} 

\begin{proof}[Proof of Lemma~\ref{lem:prob-duals}]
	To show 1) we observe that 
	$(\eta^{-1}\Phi)^*(v) 
	= \sup_{p} \inner{v,p} - \eta^{-1}\Phi(p)
	= \eta^{-1} \sup_{p} \inner{\eta v,p} - \Phi(p)
	= \eta^{-1}\Phi^*(\eta v)$.
	For 2), we note that the definition of the dual implies
	\(
		\Phi^*(v + \alpha\ones) 
		= \sup_{\mu\in\DT} \inner{\mu,v + \alpha\ones} - \Phi(\mu)
		= \sup_{\mu\in\DT} \inner{\mu,v} - \Phi(\mu) + \alpha
		= \Phi^*(v) + \alpha
	\)
	since $\inner{\mu,\ones}=1$.
	Taking derivatives of both sides gives the final part of the lemma.
\end{proof}

\begin{proof}[Proof of Lemma~\ref{lem:alt-mix}]
	By definition 
	$\Phi^*(\D\Phi(\mu) - v) 
	= \sup_{\mu'\in\DT} \inner{\mu',\D\Phi(\mu)-v} - \Phi(\mu')$ 
	and using \eqref{eq:conjugate-sup} gives
	$\Phi^*(\D\Phi(\mu)) = \inner{\mu,\D\Phi(\mu)} - \Phi(\mu)$.
	Subtracting the former from the latter gives
	\(
		\inner{\mu,\D\Phi(\mu)} - \Phi(\mu)
		-\left[
			\sup_{\mu'\in\DT} \inner{\mu',\D\Phi(\mu)-v} - \Phi(\mu')
		\right]
	\)
	which, when rearranged gives
	\(
		\inf_{\mu'\in\DT}
			\Phi(\mu') - \Phi(\mu)
			- \inner{\D\Phi(\mu), \mu'-\mu}
			+ \inner{\mu',v}
	\)
	establishing the result.

	When $\Phi = H$ -- \ie, $\Phi$ is the (negative) Shannon entropy -- we have 
	that $\D\Phi(\mu) = \log \mu + \ones$, that  
	$\Phi^*(v) = \log \sum_\theta \exp(v_\theta)$, and so
	$\D\Phi^*(v) = \exp(v) / \sum_\theta \exp(v_\theta)$, where $\log$ and
	$\exp$ are interpreted as acting point-wise on the vector $\mu$.
	By Lemma~\ref{lem:prob-duals},  
	$\Phi^*(\D\Phi(\mu)) = \Phi^*(\log\mu + \ones) = \Phi^*(\log(\mu)) + 1 = 1$
	since $\Phi^*(\log(\mu_\theta)) = \log\sum_\theta \mu_\theta = 0$.
	Similarly, 
	$\Phi^*(\D\Phi(\mu) - \ell_x(A)) 
		= \Phi^*(\log(\mu) - \ell_x(A)) + 1
		= \log\sum_\theta \mu_\theta \exp(-\ell_x(A)) + 1$.
	Substituting this into Lemma~\ref{lem:alt-mix} and applying the second part 
	of Lemma~\ref{lem:prob-duals} shows that
	$\Mix_{\ell,x}^{\eta^{-1} H}(A, \mu) 
		= -\eta^{-1} \log\sum_\theta \exp(-\eta\ell_x(A_\theta))$,
	recovering the right-hand side of the inequality in 
	Definition~\ref{def:vovk-mix}.
\end{proof}

\begin{proof}[Proof of Lemma~\ref{lem:proper-loss}]
  By eq.~\eqref{eq:conjugate-sup} we have $F^*(\D F(p)) = \inner{p,\D F(p)} - F(p)$, giving us
  \begin{align*}
    \inner{p,\ell^F(p')} - \inner{p,\ell^F(p)}
    &= \Bigl(\inner{p',\D F(p')} - F(p') - \inner{p,\D F(p')}\Bigr)\\ &\quad
    - \Bigl(\inner{p,\D F(p)} - F(p) - \inner{p,\D F(p)}\Bigr)\\
    &= D_F(p,p'),
  \end{align*}
  from which propriety follows.
\end{proof}

\begin{proof}[Proof of Lemma~\ref{lem:updates}]
By considering the Lagrangian
\(
	\Lag(\mu,a)
	= 
	\inner{\mu,\ell_{x^t}(A)} + D_\Phi(\mu, \mu^{t-1}) 
	+ \alpha (\inner{\mu,\ones}-1)
\)
and setting its derivative to zero we see that the minimizing $\mu^t$ must 
satisfy
\(
	\D\Phi(\mu^t) 
	=
	\D\Phi(\mu^{t-1}) - \ell_{x^t}(A^t) - \alpha^t\ones
\)
where $\alpha^t \in R$ is the dual variable at step $t$.
For convex $\Phi$, the functions $\D\Phi^*$ and $\D\Phi$ are inverses
\cite{Hiriart-Urruty:2001} so 
$\mu^t = \D\Phi^*(\D\Phi(\mu^{t-1}) - \ell_{x^t}(A^t) - a^t\ones)
       = \D\Phi^*(\D\Phi(\mu^{t-1}) - \ell_{x^t}(A^t))$
by the translation invariance of $\Phi^*$ (Lemma~\ref{lem:prob-duals}).
This means the constants $\alpha^t$ are arbitrary and can be ignored.
Thus, the mixture updates satisfy the relation in the lemma and summing
over $t=1,\ldots,T$ gives \eqref{eq:updates}.
\end{proof}

\begin{amrnote}{}
	Somebody (Bob?) check that my translation from the JMLR version
	is sane.
\end{amrnote}

\begin{proof}[Proof of Lemma~\ref{lem:proper-qc}]
	Let $n = |X|$ and fix an arbitrary $p\in\DX$. 
	The function $f_p(q) = \inner{p, \ell(q)}$ 
	is quasi-convex if its $\alpha$ sublevel sets 
	\(
	F_p^\alpha := \{q\in\DX \colon \inner{p,\ell(q)} \le\alpha\}
	\)
	are convex for all $\alpha\in\R$. 
	Let $g(p) := \inf_{q} f_p(q)$ and fix an arbitrary 
	$\alpha > g(p)$ so that 
	$F_p^\alpha\ne\emptyset$. 
	Let
	\(
		Q_p^\alpha := \{v\in\R^n \colon \inner{p,v} \le \alpha \} 
	\)
	so $F_p^\alpha = \{q\in\DX \colon \ell(q) \in Q_p^\alpha\}$.
	Denote by 
	\(
		h_q^\beta:=\{v\colon \inner{v, q} = \beta\}
	\)
	the hyperplane in direction $q\in\DX$ with offset
	$\beta\in\R$ and by
	\(
		H_q^\beta:=\{v\colon \inner{v, q} \ge \beta\}
	\)
	the corresponding half-space. Since $\ell$ is proper,
	its \emph{superprediction set} 
	$\mathcal{S}_\ell 
	= \{ 
		\lambda \in \R^n : \exists q\in\DX \forall x\in X \lambda_x \ge \ell_x(q)
	\}$ 
	(see \cite[Prop. 17]{Vernet:2011}) is supported at $x=\ell(q)$ by the
	hyperplane $h_{q}^{g(q)}$ and furthermore since $\mathcal{S}_\ell$ is
	convex,  $\mathcal{S}_\ell=\bigcap_{q\in\DX}
	H_q^{g(q)}$. 

	Let
	\[
		V_p^\alpha := \bigcap_{v\in\ell(\DX)\cap Q_p^\alpha}
		H_{\ell^{-1}(v)}^{g(\ell^{-1}(v))}
		=\bigcap_{q\in F_p^\alpha} 
		H_q^{g(q)}
	\]
	(see figure \ref{figure:quasi-convexity-proof}).
	Since $V_p^\alpha$ is the intersection of halfspaces it is convex.
	Note that a given half-space $H_q^{g(q)}$ is supported by
	exactly one hyperplane, namely $h_q^{g(q)}$. Thus the set of
	hyperplanes that support $V_p^\alpha $ is
	$ \{h_q^{g(q)}\colon q\in F_p^\alpha\}$
	If $u\in F_p^\alpha$ then there is a hyperplane in direction
	$u$ that supports $V_p^\alpha$ and its offset is given by
	\[
	\sigma_{V_p^\alpha}(u):=\inf_{v\in V_p^\alpha} \inner{u,v} =
	g(p)>-\infty
	\]
	whereas if $u\not\in F_p^\alpha$ then for all $\beta\in\R$, $h_u^\beta$ 
	does not
	support $V_p^\alpha$ and hence $\sigma_{V_p^\alpha}(u)=-\infty$.
	Thus we have shown 
	\[
	\left(u\not\in W_p^\alpha\right) \Leftrightarrow
	\left(\sigma_{V_p^\alpha}(u)=-\infty\right).
	\]
	Observe that
	$\sigma_{V_p^\alpha}(u)=-s_{V_p^\alpha}(-u)$ where
	$s_C(u)=\sup_{v\in C} \inner{u, v}$ is the support function of a set $C$. It
	is known \cite[Theorem 5.1]{Valentine:1964} that the ``domain of
	definition'' of a support function
	$\{u\in\R^n\colon s_C(u)<+\infty\}$ for a convex
	set $C$ is always convex. Thus $G_p^\alpha:=\{u\in\DX\colon
	\sigma_{V_p^\alpha}(u) > -\infty\}=\{u\in\R^n\colon
	\sigma_{V_p^\alpha}(u)>-\infty\}\cap\DX$ is always convex
	because it is the intersection of convex sets. Finally by
	observing that 
	\[
	G_p^\alpha=\{p\in\DX\colon
	\ell(p)\in\ell(\DX)\cap Q_p^\alpha\}=F_p^\alpha
	\]
	we have shown that $F_p^\alpha$ is convex. Since $p\in\DX$ and
	$\alpha\in\R$ were arbitrary we have thus shown that $f_p$ is
	quasi-convex for all $p\in\DX$.

\end{proof}

\begin{figure}
	\centering
	\includegraphics[width=0.8\linewidth]{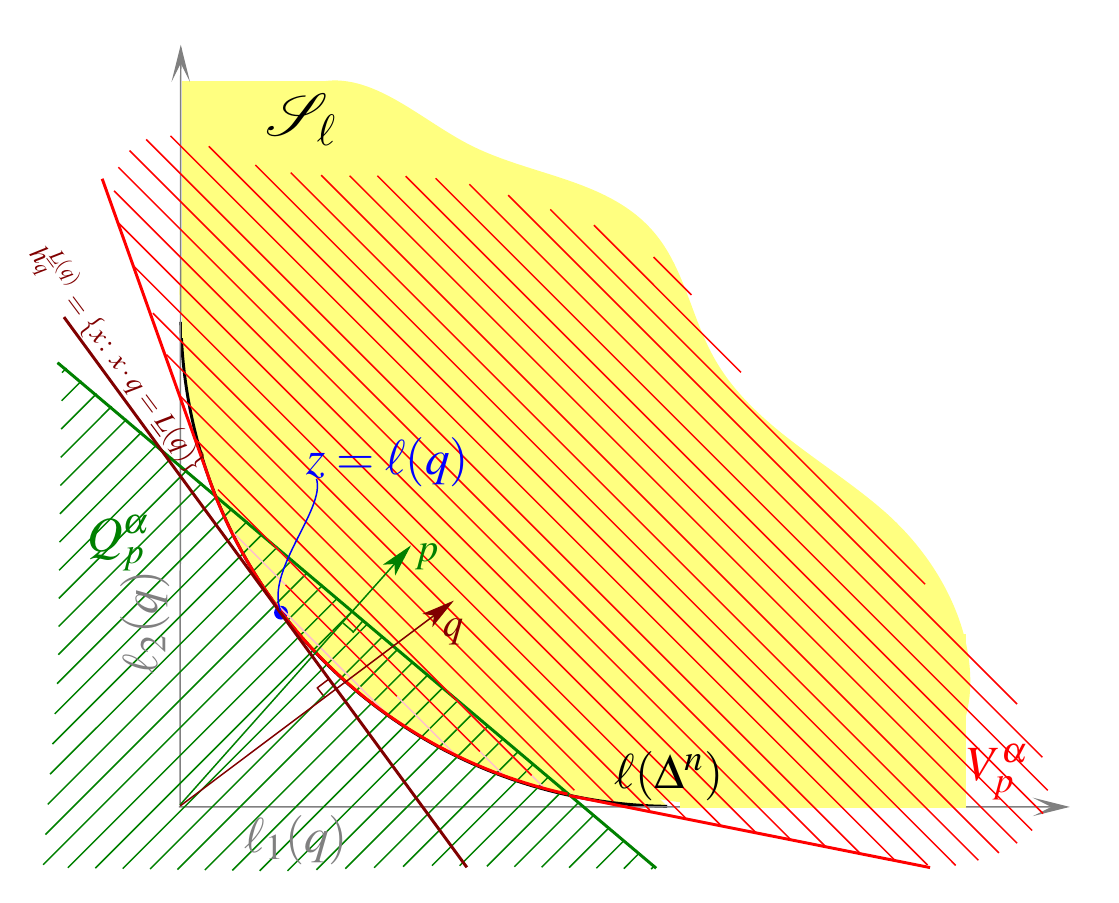}
	\caption{Visualization of construction in proof of 
		Lemma~\ref{lem:proper-qc}.\label{figure:quasi-convexity-proof}}
\end{figure}


\subsection{Proof of Theorem~\ref{thm:main}}\label{app:proof-of-main} 

\begin{proof}[Proof of Theorem~\ref{thm:main}]
Applying Lemma~\ref{lem:alt-mix} to the assumption that $\lossx$ is 
$\Phi$-mixable means that for $\mu$ equal to the updates $\mu^t$
from Definition~\ref{def:updates} and $A^t$ equal to the expert predictions
at round $t$, there must exist an $\acth^t \in \DX$ such that
\[
	\lossx_{x^t}(\acth^t)
	\le
	\Phi^*(\D\Phi(\mu^{t-1})) - \Phi^*(\D\Phi(\mu^{t-1}) - \ell_{x^t}(A^t))
\]
for all $x^t \in X$.
Summing these bounds over $t=1,\ldots,T$ gives
\begin{align}
	\sum_{t=1}^T \lossx_{x^t}(p^t)
	\le&
	\sum_{t=1}^T \Phi^*(\D\Phi(\mu^{t-1})) 
		- \Phi^*(\D\Phi(\mu^{t-1}) - \ell_{x^t}(A^t))
	\notag
	\\
	= &
	\Phi^*(\D\Phi(\mu^0)) - \Phi^*(\D\Phi(\mu^T))
	\label{eq:telescopes}\\
	= &
	\inf_{\mu'\in\DT} 
		\inner{\mu',\sum_{t=1}^T \ell_{x^T}(A^t)}
		+ D_\Phi(\mu',\mu^0)
	\label{eq:mu-T} \\
	\le&
	\inner{\mu',\sum_{t=1}^T \ell_{x^t}(A^t)} + D_\Phi(\mu',\mu^0)
	\qquad\text{ for all } \mu' \in \DT
	\label{eq:converse-needs-this}
\end{align}
Line \eqref{eq:telescopes} above is because 
$\D\Phi(\mu^t) = \D\Phi(\mu^{t-1}) - \ell_{x^t}(A^t)$ by Lemma~\ref{lem:updates} 
and the series telescopes.
Line ~\eqref{eq:mu-T} is obtained by applying \eqref{eq:gen-aa-update} from
Lemma~\ref{lem:updates} and matching equations \eqref{eq:alt-mix} and 
\eqref{eq:mixability}.
Setting $\mu' = \delta_\theta$ and noting 
$\inner{\delta_\theta,\ell(A^t)} = \lossx_{x^t}(A^t_\theta)$
gives the required result.
\end{proof}

 \subsection{Proof of Theorem~\ref{thm:mix-proper}}\label{app:mix-proper} 

We first establish a general reformulation of $\Phi$-mixability that holds for
arbitrary $\ell$ by converting the quantifiers in the definition of $\Phi$-mixability from Lemma~\ref{lem:alt-mix} for $\ell$ into an expression involving 
infima and suprema. 
We then further refine this by assuming $\ell = \ell^F$ is proper (and thus quasi-convex) and has Bayes risk $F$.
\begin{align}
	& 
	\inf_{A,\mu} \sup_{\acth} \inf_{x} \;
	\Phi^*(\D\Phi(\mu)) - \Phi^*(\D\Phi(\mu) - \ell_x^F(A)) - \ell^F_x(\acth)
	\ge 0
	\notag \\ 
	\iff
	& 
	\inf_{A,\mu} \sup_{\acth} \inf_{p} \;
	\inner{
		p, 
		\left\{
			\Phi^*(\D\Phi(\mu))-\Phi^*(\D\Phi(\mu)-\ell_x^F(P))
		\right\}_x
	} - \inner{p, \ell^F_x(\phat)}
	\ge 0
	\label{eq:first-objective}
\end{align}
where the term in braces is a vector in $\R^X$.
The infimum over $x$ is switched to an infimum over distributions
over $p \in\DX$ because the optimization over $p$ will be achieved on the 
vertices of the simplex as it is just an average over random variables over $X$.

From here on we assume that $\ell = \ell^F$ is proper and adjust our notation
to emphasis that actions $\acth = \phat$ and $A = P$ are distributions.
Note that the new expression is linear -- and therefore convex in $p$ -- and,
by Lemma~\ref{lem:proper-qc}, we know $\ell^F$ is quasi-convex and so the function being optimized in \eqref{eq:first-objective} is quasi-concave in 
$\phat$.
We can therefore apply Sion's theorem to swap $\inf_p$ and $\sup_{\phat}$ 
which means $\ell^F$ is $\Phi$-mixable if and only if
\begin{align*}
	&
	\inf_{P,\mu} \inf_{p} \sup_{\phat} \;
	\inner{
		p, 
		\left\{
			\Phi^*(\D\Phi(\mu))-\Phi^*(\D\Phi(\mu)-\ell_x^F(P))
		\right\}_x
	} - \inner{p, \ell^F_x(\phat)}
	\ge 0 \notag
	\\
	\iff
	&
	\inf_{P,\mu} \inf_p \;
		\Phi^*(\D\Phi(\mu)) 
		- \inner{p, \left\{ \Phi^*(\D\Phi(\mu) - \ell^F_x(P)) \right\}_x)} 
		+ F(p) \ge 0 
	\\
	\iff
	&
	\inf_{P,\mu} \;
		\Phi^*(\D\Phi(\mu)) 
		- F^*(\left\{ \Phi^*(\D\Phi(\mu) - \ell^F_x(P)) \right\}_x)
	\ge 0
\end{align*}
The second line above is obtained by recalling that, by the definition of 
$\ell^F$, its Bayes risk is $F$.
We now note that the inner infimum over $p$ passes through $\Phi^*(\D\Phi(\mu))$ 
so that the final two terms are just the convex dual for $F$ evaluated at 
$\left\{\Phi^*(\D\Phi(\mu) - \ell^F_x(P))\right\}_x$.
Finally, by translation invariance of $F^*$ we can pull the $\Phi^*(\pi^*)$
term inside $F^*$ to simplify further so that the loss $\ell^F$ with Bayes risk
$F$ is $\Phi$-mixable if and only if
\begin{equation*}
	\inf_{P,\mu} \;
		- F^*\left(
			\left\{ \Phi^*(\D\Phi(\mu) - \ell^F_x(P)) \right\}_x 
			- \Phi^*(\D\Phi(\mu))\ones
		\right)
	\ge 0.
\end{equation*}
Applying Lemma~\ref{lem:proper-loss} to write $\ell^F$ in terms of $F$ and
passing the sign through the infimum and converting it to a supremum gives the
required result.


\subsection{Proof of Theorem~\ref{thm:legendre}} \label{app:legendre} 

\begin{arfnote}{}
  If we keep $M(\eta)$ above, just back-reference (but keep the
  equation for readability)
\end{arfnote}
We will make use the following formulation of mixability,
\begin{equation}
  M(\eta) := \inf_{A\in\acts,\,\pi\in\DT} \;\sup_{\acth \in \acts} \;\inf_{\mu\in\DT,\,x\in X} \;
  \inner{\mu, \ell_x(A)} + \frac 1 \eta D_\Phi(\mu,\pi) - \ell_x(\acth), 
  \label{eq:three-term-mixability}
\end{equation}
so that $\ell$ is $\Phi_\eta$-mixable if and only if $M(\eta) \geq 0$.

We call a loss $\ell$ \emph{nontrivial} if there exist $x^*,x'$ and $a^*,a'$ such that
\begin{equation}
  \label{eq:phi-mix-nontrivial-loss}
  a' \in \argmin\{\ell_{x^*}(a) : \ell_{x'}(a) = \inf_{a\in\acts} \ell_{x'}(a)\}
  \text{ and } 
  \inf_{a\in\acts}\ell_{x^*}(a) = \ell_{x^*}(a^*) < \ell_{x^*}(a')~.
\end{equation}

Intuitively, this means that there exist distinct actions which are optimal for different outcomes $x^*,x'$.  Note that in particular, among all optimum actions for $x'$, $a'$ has the lowest loss on $x^*$.

\begin{arfnote}{}
  The following lemma requires the Bayes risk of $\ell$ to be strictly convex, but after an email thread with Bob and Nishant, I think I see why (a) it shouldn't be necessary, and even (b) how to potentially relax the assumption to something like ``not linear''.  See inline notes below for some ideas.
\end{arfnote}

\begin{lemma}
  \label{lem:phi-mix-bayes-risk}
  Suppose $\ell$ has a strictly concave Bayes risk $L$.  Then given any distinct $\mu^*,\mu'\in\DT$, there is some $A\in\acts$ and $x^*,x'\in X$ such that for all $\acth\in\acts$ we have at least one of the following:
\begin{equation}
  \label{eq:phi-mix-fork}
  \inner{\mu^*, \ell_{x^*}(A)} < \ell_{x^*}(\acth)~,
  \quad
  \inner{\mu', \ell_{x'}(A)} < \ell_{x'}(\acth)~.
\end{equation}
\end{lemma}
\begin{proof}
  Let $\theta^*$ be an expert such that $\alpha := \mu^*_{\theta^*} > \mu'_{\theta^*} =: \beta$, which exists as $\mu^*\neq\mu'$.  Pick arbitrary $x^*,x'\in X$ and let $p^*,p'\in\DX$ with support only on $\{x^*,x'\}$ and $p^*_{x^*} = \alpha/(\alpha+\beta)$, 
$p'_{x^*} = (1-\alpha)/(2-\alpha-\beta)$. 
Now let $a^* = \argmin_{a\in\acts} \E{x\sim p^*}{\ell_x(a)}$,  $a' = \argmin_{a\in\acts} \E{x\sim p'}{\ell_x(a)}$, and set $A$ such that $A_{\theta^*} = a^*$ and $A_\theta = a'$ for all other $\theta\in\Theta$.

\begin{arfnote}{}
  Here I take a uniform mixture of the two inequalities.  I think if you take a nonuniform mixture, and play around with the choice of outcomes, you can essentially get any line segment $[p',p^*]$ in the simplex $\DX$, and maybe we can even set $\bar p$ to be any point in the interior.  If so, then the argument below would only require $L$ to be strictly concave \textbf{somewhere}; otherwise, it must be linear.
\end{arfnote}
Now suppose there is some $\acth\in\acts$ violating eq.~\eqref{eq:phi-mix-fork}.  Then in particular,
\begin{align*}
  \tfrac{1}{2} \left(\ell_{x^*}(\acth) + \ell_{x'}(\acth)\right)
  & \le \tfrac 1 2 \left(\inner{\mu^*, \ell_{x^*}(A)} + \inner{\mu', \ell_{x'}(A)}\right) \\
  & = \tfrac 1 2 \left(\alpha\ell_{x^*}(a^*) + (1-\alpha)\ell_{x^*}(a') + \beta\ell_{x'}(a^*) + (1-\beta)\ell_{x'}(a')\right) \\
  & = \tfrac{\alpha+\beta}{2}\left(\tfrac{\alpha}{\alpha+\beta}\ell_{x^*}(a^*) + \tfrac{\beta}{\alpha+\beta}\ell_{x'}(a^*)\right) + \tfrac{2-\alpha-\beta}{2}\left(\tfrac{1-\alpha}{2-\alpha-\beta}\ell_{x^*}(a') + \tfrac{1-\beta}{2-\alpha-\beta}\ell_{x'}(a')\right) \\
  & = \tfrac{\alpha+\beta}{2} L(p^*) + \left(1 - \tfrac{\alpha+\beta}{2}\right) L(p')~.
\end{align*}
Letting $\bar p \in \DX$ with $\bar p_{x^*} = \bar p_{x'} = 1/2$, observe that $\bar p = \tfrac{\alpha+\beta}{2} p^* + (1 - \tfrac{\alpha+\beta}{2}) p'$.  But by the above calculation, we have $L(\bar p) \le \tfrac{\alpha+\beta}{2} L(p^*) + (1 - \tfrac{\alpha+\beta}{2}) L(p')$, thus violating strict concavity of $L$.
\end{proof}

\subsubsection*{Non-Legendre$\implies$no nontrivial mixable $\ell$ with strictly convex Bayes risk:}
To show that no non-constant $\Phi$-mixable losses exist, we must 
exhibit a $\pi \in \DT$ and an $A \in \acts$ such that for all $\acth \in \acts$
we can find a $\mu\in\DT$ and $x \in X$ satisfying
\(
	\inner{\mu, \ell_x(A)} + \frac 1 \eta D_\Phi(\mu,\pi) - \ell_x(\acth)  < 0.
\)
Since $\Phi$ is non-Legendre it must either (1) fail strict convexity, or (2) have a point on the boundary with bounded derivative; we will consider each case separately.

\paragraph{(1)}
Assume that $\Phi$ is not strictly convex; then we have some $\mu^*\neq\mu'$ such that $D_\Phi(\mu^*,\mu') = 0$.  By Lemma~\ref{lem:phi-mix-bayes-risk} with these two distributions, we have some $A$ and $x^*,x'$ such that for all $\acth$, either (i) $\inner{\mu^*, \ell_{x^*}(A)} < \ell_{x^*}(\acth)$ or (ii) $\inner{\mu', \ell_{x'}(A)} < \ell_{x'}(\acth)$.  We set $\pi=\mu'$; in case (i) we take $\mu=\mu^*$ and $x=x^*$, and in (ii) we take $\mu=\mu'$ and $x=x'$, but as $\frac 1 \eta D_\Phi(\mu,\pi) = 0$ in both cases, we have $M(\eta)<0$ for all $\eta$.

\paragraph{(2)}
Now assume instead that we have some $\mu'$ on the boundary of $\DT$ with bounded $\|\nabla\Phi(\mu')\| = C < \infty$.
Because $\mu'$ is on the boundary of $\DT$ there is at least one expert 
$\theta^* \in \Theta$ for which $\mu'_{\theta^*} = 0$.  
Pick $x^*,x',a^*,a'$ from the definition of nontrivial, eq.~\eqref{eq:phi-mix-nontrivial-loss}.  In particular, note that $\ell_{x^*}(a^*) < \ell_{x^*}(a')$.
Let $\pi = \mu'$ and $A \in \acts$ such that $A_{\theta^*} = a^*$ and $A_\theta = a'$ for all other $\theta$.  

Now suppose $\acth\in\acts$ has $\ell_{x'}(\acth) > \ell_{x'}(a')$.  Then taking $\mu=\pi$ puts all weights on experts predicting $a'$ while keeping $D_\Phi(\mu,\pi)=0$, so choosing $x=x'$ gives $M(\eta) < 0$ for all $\eta$.  Otherwise, $\ell_{x'}(\acth) = \ell_{x'}(a')$, which by eq.~\eqref{eq:phi-mix-nontrivial-loss} implies $\ell_{x^*}(\acth) \ge \ell_{x^*}(a')$.  Let $\mu^\alpha = \pi + \alpha(\delta_{\theta^*}-\pi)$, where $\delta_{\theta^*}$ denotes the point distribution on $\theta^*$.  Calculating, we have
\begin{align*}
  M(\eta) 
  &= \inner{\mu^\alpha, \ell_{x^*}(A)} + \tfrac 1 \eta D_\Phi(\mu^\alpha,\pi) - \ell_{x^*}(\acth) \\
  &= (1-\alpha)\ell_{x^*}(a') + \alpha \ell_{x^*}(a^*) + \tfrac 1 \eta D_\Phi(\mu^\alpha,\pi) - \ell_{x^*}(\acth) \\
  &\le (1-\alpha)\ell_{x^*}(\acth) + \alpha \ell_{x^*}(a^*) + \tfrac 1 \eta D_\Phi(\mu^\alpha,\pi) - \ell_{x^*}(\acth) \\
  &= \alpha (\ell_{x^*}(a^*) - \ell_{x^*}(\acth)) + \tfrac 1 \eta D_f(\alpha,0),
\end{align*}
where $f(\alpha) = \Phi(\mu^\alpha) = \Phi(\pi + \alpha(\delta_{\theta^*}-\pi))$.  As $\nabla_\pi\Phi$ is bounded, so is $f'(0)$.  Now as $\lim_{\epsilon\to 0} D_f(x+\epsilon,x)/\epsilon = 0$ for any scalar convex $f$ with bounded $f'(x)$ (see e.g.~\cite[Theorem 24.1]{Rockafellar:1997} and~\cite{abernethy2012characterization}), we see that for any $c>0$ we have some $\alpha>0$ such that $D_f(\alpha,0) < c \alpha$.  Taking $c = \eta (\ell_{x^*}(\acth) - \ell_{x^*}(a^*)) > 0$ then gives $M(\eta)<0$.

\subsubsection*{Legendre$\implies$ $\exists$ mixable $\ell$:}

\begin{arfnote}{}
  Needs fixing up still!  I took $a$ to $p$ and $A$ to $P$, which should be remedied.
\end{arfnote}

Assuming $\Phi$ is Legendre, we need only show that some non-constant $\ell$ is
$\Phi$-mixable.
As $\nabla_\pi\Phi$ is infinite on the boundary, $\pi$ must be in the relative interior of $\Delta_\Theta$; otherwise $D_\Phi(\mu,\pi) = \infty$ for $\mu\neq\pi$.

Take $\acts = \DX$ and $\ell(p,x) = \|p-\delta_x\|^2$ to be the 2-norm squared 
loss.  
Now for all $\mu$ in the interior of $\Delta_\Theta$ and 
$P \in \DX^\Theta$, we have 
$\inner{\mu, \ell_x(P)} 
	= \sum_\theta \mu_\theta\|P_\theta-\delta_x\|^2 
	\geq \|\bar p-\delta_x\|^2$ by convexity, 
where $\bar p = \sum_\theta \mu_\theta P_\theta$. 
In fact, as $\mu$ is in the interior, this inequality is strict, 
and remains so if replace $\mu$ by $\mu'$ with $\|\mu'-\mu\|<\epsilon$ 
for some $\epsilon$ sufficiently small.
Now for all $\mu,P$ the algorithm can take $\phat=\bar p$, and we can 
always choose $\eta = \inf_{x,\mu':\|\mu'-\mu\|=\epsilon} D_\Phi(\mu',\mu) / (\epsilon \ell_{\mathrm{max}}) > 0$, so either $\|\mu-\pi\|<\epsilon$ in which case we are fine by the above, or $\mu$ is far enough away that the $D_\Phi$ term dominates the algorithm's loss.  (Here $\ell_{\mathrm{max}}$ is just $\max_{p,x} \ell_x(p)$, which is bounded, and $D_\Phi(\mu',\mu) > 0$ as $\Phi$ is strictly convex.)  So if $\Phi$ is Legendre, squared loss is $\Phi$-mixable.

 \end{document}